\documentclass{article}

\PassOptionsToPackage{numbers}{natbib}

\usepackage[preprint]{nips_2018}

\usepackage[utf8]{inputenc} 
\usepackage[T1]{fontenc}    
\usepackage{hyperref}       
\usepackage{url}            
\usepackage{booktabs}       
\usepackage{amsfonts}       
\usepackage{nicefrac}       
\usepackage{microtype}      

\usepackage{graphicx} 
\usepackage[caption=false]{subfig}

\usepackage{algorithm}
\usepackage{algorithmic}
\usepackage{nccmath}

\usepackage{url}

\usepackage{bm}

\newcommand{\req}[1]{Eq.~(\ref{#1})}
\newcommand{\rfig}[1]{Fig.~\ref{#1}}
\newcommand{\rtab}[1]{Tab.~\ref{#1}}
\usepackage{amsmath,cases}
\usepackage{amsthm}
\newtheorem{theorem}{Theorem}

\newtheorem{definition}{Definition}

\usepackage{pifont}

\title{Sigsoftmax: Reanalysis of the Softmax Bottleneck}
\author{
  Sekitoshi~Kanai\\ 
NTT Software Innovation Center, Keio Univ.\\
  \texttt{kanai.sekitoshi@lab.ntt.co.jp} \\
  \And
Yasuhiro~Fujiwara\\ 
NTT Software Innovation Center\\
  \texttt{ fujiwara.yasuhiro@lab.ntt.co.jp} \\
\And
Yuki~Yamanaka\\ 
NTT Secure Platform Laboratories\\
  \texttt{ yamanaka.yuki@lab.ntt.co.jp} \\
   \And
 Shuichi~Adachi \\
 Keio Univ.  \\
  \texttt{adachi.shuichi@appi.keio.ac.jp} \\
}

\begin{document}

\maketitle
\begin{abstract}
Softmax is an output activation function
for modeling categorical probability distributions in many applications of deep learning.
However, a recent study revealed that softmax can be a bottleneck of representational
capacity of neural networks in language modeling (the softmax bottleneck).
In this paper, we propose an output activation function for breaking the softmax bottleneck without additional parameters.
We re-analyze the softmax bottleneck from the perspective of the output set of log-softmax and
identify the cause of the softmax bottleneck.
On the basis of this analysis, we propose sigsoftmax, which
 is composed of a multiplication of an exponential function and sigmoid function. 
 Sigsoftmax can break the softmax bottleneck.
 The experiments on language modeling demonstrate that sigsoftmax
and mixture of sigsoftmax outperform softmax and mixture of softmax, respectively.
\end{abstract}

\section{Introduction}
\label{Intro}
Deep neural networks are used in many recent applications such as image recognition \cite{krizhevsky2012imagenet,resnet},  speech recognition \cite{graves2013speech}, and
natural language processing \cite{tomas,NIPS2014_5346, cho-al-emnlp14}.
High representational capacity and generalization performance of deep neural networks
are achieved by many layers, activation functions and regularization methods \cite{relu, resnet, srivastava2014dropout,ioffe2015batch,goodfellow2016deep}.
Although various model architectures are built in the above applications,
 softmax is commonly used as an output activation function
for modeling categorical probability distributions
 \cite{softmax,goodfellow2016deep,resnet, tomas, NIPS2014_5346, cho-al-emnlp14, graves2013speech}.
 For example, in language modeling, softmax is employed for representing
  the probability of the next word over the vocabulary in a sentence.
When using softmax, we train the model by minimizing negative log-likelihood
 with a gradient-based optimization method.
 We can easily calculate the gradient of negative log-likelihood with softmax, and it is numerically stable \cite{bridle1990training,softmax}.

 Even though softmax is widely used, few studies have attempted
  to improve its modeling performance \cite{chen2017noisy,de2015exploration}. 
  This is because deep neural networks with softmax are believed to have a universal approximation property.
However, \citet{softmaxbottle} recently revealed that softmax can be
 a bottleneck of representational capacity in language modeling.
They showed that
 the representational capacity of the softmax-based model is restricted by the length of the hidden vector in the output layer.
In language modeling, the length of the hidden vector is much smaller than
the vocabulary size. 
As a result, the softmax-based model cannot completely learn the true probability distribution, and this is called the softmax bottleneck.
 For breaking the softmax bottleneck, \citet{softmaxbottle} proposed mixture of softmax (MoS) that
mixes the multiple softmax outputs.
However, this analysis of softmax does not explicitly show why 
softmax can be a bottleneck. 
Furthermore, MoS is an additional layer or mixture model rather than an alternative
 activation function to softmax: MoS has learnable parameters and hyper-parameters.

In this paper, we propose a novel output activation function for breaking the softmax bottleneck without additional parameters.
We re-analyze the softmax bottleneck from the point of view of the output set (range) of a function and show why softmax can be a bottleneck.
This paper reveals that (i) the softmax bottleneck occurs because softmax uses only exponential functions for nonlinearity and 
(ii) the range of log-softmax is a subset of the vector space whose dimension depends on the dimension of the input space.
As an alternative activation function to softmax, 
we explore the output functions composed of rectified linear unit (ReLU) and sigmoid functions.
In addition, we propose {\it sigsoftmax}, which
 is composed of a multiplication of an exponential function and sigmoid function. 
Sigsoftmax has desirable properties for output activation functions, e.g.,
the calculation of its gradient is numerically stable.
More importantly, sigsoftmax can break the softmax bottleneck,
and the range of softmax can be a subset of that of sigsoftmax.
Experiments in language modeling demonstrate that
sigsoftmax can break the softmax bottleneck and outperform softmax.
In addition, mixture of sigsoftmax outperforms MoS.
\section{Preliminaries}
\subsection{Softmax}
\label{softmax subsec}
Deep neural networks use softmax in learning categorical distributions.
For example, in the classification, a neural network uses softmax 
 to learn the probability distribution over $M$ classes $\bm{y}\in \bm{R}^M$ conditioned on the input $\bm{x}$ as $P_{\bm{\theta}}(\bm{y}|\bm{x})$ 
where $\bm{\theta}$ is a parameter.
Let $\bm{h}(\bm{x})\in \bm{R}^d$ be a hidden vector and
 $\bm{W}\in \bm{R}^{M\times d}$ be a weight matrix in the output layer, the
output of softmax $\bm{f}_s(\cdot)$ represents the conditional probability of the $i$-th class as follows:
\begin{align}
\textstyle
P_{\bm{\theta}}(y_i|\bm{x})=
\left[\bm{f}_s(\bm{W}\bm{h}(\bm{x}))\right]_i
=\frac{\mathrm{exp}(\left[\bm{W}\bm{h}(\bm{x})\right]_i)}
{\sum_{m=1}^{M} \mathrm{exp}(\left[\bm{W}\bm{h}(\bm{x})\right]_m)},
\label{sbeq}
\end{align}
where $\left[\bm{f}_s\right]_i$ represents the $i$-th element of $\bm{f}_s$. 
We can see that each element of $\bm{f}_s$ is bounded from zero to one
 since the output of exponential functions is non-negative in \req{sbeq}. 
The summation of all elements of $\bm{f}_s$ is obviously one.
From these properties, we can  regard output of the softmax trained by minimizing negative
log-likelihood as a probability \cite{softmax,memisevic2010gated}.
If we only need the most likely label,
we can find such a label 
by comparing elements of $\bm{Wh}(\bm{x})$ without the calculations of softmax $\bm{f_s}(\bm{Wh}(\bm{x}))$
once we have trained the softmax-based model.
This is because exponential functions in softmax are monotonically increasing.

To train the softmax-based models, negative log-likelihood
(cross entropy) is used as a loss function.
Since the loss function is minimized by stochastic gradient descent (SGD),
 the properties of the gradients of functions are very important \cite{relu, pascanu2013difficulty,glorot2010understanding}.
One advantage of softmax is that
 the gradient of log-softmax is easily calculated as follows \cite{bridle1990training,softmax,bishop1995neural,de2015exploration}: 
\begin{align}
\textstyle
\frac{\partial\left[\mathrm{log}\bm{f}_s(\bm{z})\right]_i}{\partial z_j}
&\textstyle =
\begin{cases}\textstyle
1-\left[\bm{f}_s(\bm{z})\right]_j~~~\mathrm{if}~j=i,\\
\textstyle
-\left[\bm{f}_s(\bm{z})\right]_j~~~\mathrm{if}~j\neq i,
\end{cases}
\label{GradSoft}
\end{align}
where $\bm{z}=\bm{W}\bm{h}(\bm{x})$.
Whereas the derivative of the logarithm can cause a division by zero since
 $\frac{\mathrm{d} \mathrm{log}(z)}{\mathrm{d}z}\!=\!\frac{1}{z}$, 
 the derivative of log-softmax cannot.
As a result, softmax is numerically stable.
\subsection{Softmax bottleneck}
\label{sbsec1}
In recurrent neural network (RNN) language modeling, given a corpus of tokens $\bm{Y}=(Y_1,\dots,Y_T)$, 
the joint probability $P(\bm{Y})$ is factorized as $P(\bm{Y})=\prod_t P(Y_t|Y_{<t})=\prod_t P(Y_t|X_t)$, where
$X_t=Y_{<t}$ is referred to as the context of the conditional probability.
Output of softmax $\bm{f}_s(\bm{W}\bm{h}(X_t))$ learns $P(Y_t|X_t)$ where (a) $\bm{h}(X_t)\in \bm{R}^d$
 is the hidden vector corresponding to the context $X_t$ and (b) $\bm{W}$ is a weight matrix in the output layer (embedding layer).
A natural language is assumed as a finite set of pairs of $x_t$ and $P^*(Y|x_t)$ as 
$\mathcal{L}=\left\{ (x_1,P^*(Y|x_1)),\dots,(x_N,P^*(Y|x_N)) \right\}$, where $N$ is the number of possible contexts. 
The objective of language modeling is to learn a model distribution $P_{\bm{\theta}}(Y|X)$ parameterized 
by $\bm{\theta}$ to match the true data distribution $P^*(Y|X)$.
Note that upper- and lower-case letters are used for variables and constants, respectively, in this section.
Under the above assumptions,
let $y_1,\dots, y_M$ be $M$ possible tokens in the language $\mathcal{L}$,
 the previous study of \citet{softmaxbottle} considers
 the following three matrices:
\begin{align}
\textstyle
\bm{H}_{\bm{\theta}}\!=\!\left[\!\!\!\begin{array}{c}
\bm{h}(x_1)^T\\
\bm{h}(x_2)^T\\
\vdots\\
\bm{h}(x_N)^T
\end{array}\!\!\!\!\right]\!\!, 
\bm{W},
\bm{A}=\!\left[\!\!\!\begin{array}{cccc}
\mathrm{log}P^*(y_1|x_1),&\!\mathrm{log}P^*(y_2|x_1),&\!\dots&\mathrm{log}P^*(y_M|x_1)\\
\mathrm{log}P^*(y_1|x_2),&\!\mathrm{log}P^*(y_2|x_2),&\!\dots&\mathrm{log}P^*(y_M|x_2)\\
\vdots&\!\vdots&\ddots&\!\vdots\\
\mathrm{log}P^*(y_1|x_N),&\!\mathrm{log}P^*(y_2|x_N),&\!\dots&\mathrm{log}P^*(y_M|x_N)
\end{array}\!\!\!\!\right]\!\!.
\label{A}
\end{align}
$\bm{H}_{\bm{\theta}}\in \bm{R}^{N\times d}$ is a matrix composed of the hidden vectors, $\bm{W} \in \bm{R}^{M\times d}$ is a weight matrix, and 
 $\bm{A}\in\bm{R}^{M\times N}$ is a matrix composed of the log probabilities of the true distribution.
By using these matrices, the rank of
$\bm{H}_{\bm{\theta}}\bm{W}^T$ should be greater than or equal to $\mathrm{rank}(\bm{A})-1$
so that the softmax-based model completely learns $\mathcal{L}$ \cite{softmaxbottle}.
However, the rank of $\bm{H}_{\bm{\theta}}\bm{W}^T$ is at most $d$ 
if any functions $\mathcal{U}$ are used for $\bm{H}_{\bm{\theta}}$ and $\bm{W}$. 
Therefore, if we have $d<\mathrm{rank}(\bm{A})-1$, softmax can be the bottleneck of representational capacity
 as shown in the following theorem: 
\begin{theorem}[Softmax Bottleneck \citep{softmaxbottle}]
If $d<\mathrm{rank}(\bm{A})-1$, for any function family $\mathcal{U}$ and any model
parameter $\bm{\theta}$, there exists a context $x$ in $\mathcal{L}$ such that
$P_{\bm{\theta}}(Y|x)\neq P^{*} (Y|x)$.
\label{sb1}
\end{theorem}
This theorem shows that the length of the hidden vector in the output layer determines
 the representational power of RNN with softmax.
In language modeling, the rank of $\bm{A}$ can be extremely high since contexts can vary and vocabulary size $M$ is
much larger than $d$.
Therefore, the softmax can be the bottleneck of the representational power.
\subsection{Mixture of softmax}
A simple approach to improving the representational capacity is to use a weighted sum of the several models.
In fact, \citet{softmaxbottle} use this approach for breaking the softmax bottleneck.
As the alternative to softmax, they propose the mixture of softmax (MoS), which
is the weighted sum of $K$ softmax functions:
\begin{align}
\textstyle
P_{\bm{\theta}}(y_i|x)=\sum_{k=1}^{K}\pi(x,k)\frac{\mathrm{exp} (\left[\bm{W}\bm{h}(x,k)\right]_i)}{\sum_{m=1}^{M}\mathrm{exp} (\left[\bm{W}\bm{h}(x,k)\right]_m)},
\end{align}
where $\pi(x,k)$ is the prior or mixture weight of the $k$-th component, and $\bm{h}(x,k)$ is the $k$-th context
vector associated with the context $x$.
Let $\bm{h}'(x)$ be input of MoS for the context $x$.
The priors and context vectors are parameterized as 
$\pi({x},k)=\frac{\mathrm{exp}(\bm{w}_{\pi,k}^T\bm{h}'(x))}{\sum_{k'=1}^K\mathrm{exp}(\bm{w}_{\pi, k'}^T\bm{h}'(x))}$
and $\bm{h}({x},k)=\mathrm{tanh}(\bm{W}_{h,k}\bm{h}'(x))$, respectively.
MoS can break the softmax bottleneck since the rank of the approximate $\bm{A}$ can be arbitrarily large \cite{softmaxbottle}.
Therefore, language modeling with MoS performs better than that with softmax. 
However, in this method, the number of mixtures $K$ is the hyper-parameter which needs to be tuned.
In addition, weights $\bm{W}_{h,k}$ and $\bm{w}_{\pi,k}$ are additional parameters.
Thus, MoS can be regarded as an additional layer or mixing technique rather than the improvement of the activation
function.
\subsection{Related work}
Previous studies proposed 
 alternative functions to softmax \cite{de2015exploration, mohassel2017secureml,ollivier2013riemannian}. 
The study of \citet{de2015exploration} explored spherical family functions: the spherical softmax and Taylor softmax.
 They showed that these functions do not outperform softmax
when the length of an output vector is large.
 In addition, the spherical softmax has a hyper-parameter
that should be carefully tuned for numerical stability reasons \cite{de2015exploration}.
On the other hand, the Taylor softmax might suffer from the softmax bottleneck since it approximates softmax.
\citet{mohassel2017secureml} proposed a ReLU-based alternative function to softmax for 
privacy-preserving machine learning since softmax is expensive to compute inside a
secure computation. However, it leads to a division by zero
 since all outputs of ReLUs frequently become zeros and the denominator for normalization becomes zero.
Several studies improved the efficiency of softmax \cite{joulin2017efficient,shim2017svd,aueb2016one,martins2016softmax}.
However they did not improve the representational capacity.
\section{Proposed method}
\subsection{Reanalysis of the softmax bottleneck}
\label{sbsec2}
The analysis of the softmax bottleneck \cite{softmaxbottle} is based on matrix factorization 
and reveals that the rank of $\bm{H_{\theta}W_\theta}^T$
needs to be greater than or equal to $\mathrm{rank}(A)-1$.
Since the rank of $\bm{H_{\theta}W_\theta}^T$ becomes the length of the hidden vector in the output layer,
 the length of the hidden vector determines the representational power as described in Sec.~\ref{sbsec1}.
However, this analysis does not explicitly reveal the cause of the softmax bottleneck.
To identify the cause of the softmax bottleneck, 
we re-analyze the softmax bottleneck from the perspective of the range of log-softmax
because it should be large enough to approximate the true log probabilities.

Log-softmax is a logarithm of softmax and is used in training of deep learning as mentioned in Sec.~\ref{softmax subsec}.
By using the notation in Sec.~\ref{softmax subsec}, log-softmax $\mathrm{log}(\bm{f}_s(\bm{z}))$ can be represented as 
$\left[\mathrm{log}(\bm{f}_s(\bm{z}))\right]_i\!=\!\mathrm{log}\left(\!\frac{\mathrm{exp}(z_i)}
{\sum_{m=1}^{M} \!\mathrm{exp}(z_m)}\right)
\!=\!z_i-\mathrm{log}(\sum_{m=1}^{M} \mathrm{exp}(z_m))$.
This function can be expressed as 
\begin{align}
\textstyle 
\mathrm{log}\left(\bm{f}_s(\bm{z})\right)=\bm{z}-\mathrm{log}(\sum_{m=1}^{M} \mathrm{exp}(z_m) )\bm{1},
\label{logsoft}
\end{align}
where $\bm{1}$ is the vector of all ones. 
To represent various log probability distributions $\mathrm{log}(P^*(\bm{y}|\bm{x}))$, 
the range of $\log\left(\bm{f}_s(\bm{z})\right)\in \bm{R}^M$ should be sufficiently large.
Therefore, we investigate the range of $\mathrm{log}\left(\bm{f}_s(\bm{z})\right)$.
We assume that the hidden vector $\bm{h}$ in the output layer can be an arbitrary vector in $\bm{R}^d$ where $d\leq M$,
 and the weight matrix $\bm{W}\in\bm{R}^{M\times d}$ is
 the full rank matrix; the rank of $\bm{W}$ is $d$.\footnote{If neural networks have the universal approximation property, 
 $\bm{h}$ can be an arbitrary vector in $\bm{R}^d$. If not, the input space is a subset of a $d$ dimensional vector space,
  and the range of log-softmax is still a subset of a $d+1$ dimensional vector space.
 When $\mathrm{rank}(\bm{W})<d$, we can examine the range of log-softmax in the same way by replacing $d$ with $\mathrm{rank}(\bm{W})$.
 If a bias is used in the output layer, the dimension of $S$ can be $d+1$.
 }
 Under these assumptions, 
the input vector space of softmax $S$ ($\bm{z}\in S$) is a $d$ dimensional vector space, and
 we have the following theorem:
 \begin{theorem}
Let $S\subseteq \bm{R}^M$ be the $d$ dimensional vector space and $\bm{z}\in S$ be input of log-softmax,
every range of the log-softmax $\{\mathrm{log}(\bm{f}_s(\bm{z}))|\bm{z} \in S\}$ is a subset of the $d+1$ dimensional vector space.\label{sb2}
\end{theorem}
\begin{proof}
The input of log-softmax $\bm{z}=\bm{W}\bm{h}$ can be represented by $d$ singular vectors of $\bm{W}$
since the rank of $\bm{W}$ is $d$.
In other words, the space of input vectors $\bm{z}$ is spanned by $d$ basis vectors.
Thus, the input vector space $\{\bm{z}|\bm{z}\in S\}$ is represented as $\{\sum_{l=1}^d k^{(l)}\bm{u}^{(l)} |k^{(l)} \in \bm{R} \}$ where
$\bm{u}^{(l)}\in\bm{R}^M$ for $l=1,\dots, d$ are linearly independent vectors and $k^{(l)}$ are their coefficients.
 From \req{logsoft}, by using $\bm{u}^{(l)}$ and $k^{(l)}$,
the range of log-softmax $\{\mathrm{log}(\bm{f}_s(\bm{z}))|\bm{z} \in S\}$ becomes 
\begin{align}
\textstyle
\left\{\mathrm{log}(\bm{f}_s(\bm{z}))\right|\bm{z} \in S\}=\{\sum_{l=1}^d k^{(l)}\bm{u}^{(l)}-c(\sum_{l=1}^d k^{(l)}\bm{u}^{(l)})\bm{1}|k^{(l)}\in \bm{R}\},
\end{align} 
where $c(\sum_{l=1}^d k^{(l)}\bm{u}^{(l)})=\mathrm{log}(\sum_{m=1}^{M} \mathrm{exp}(\left[\sum_{l=1}^d k^{(l)}\bm{u}^{(l)} \right]_m))$.
This is the linear combination of $d$ linearly independent vectors $\bm{u}^{(l)}$ and $\bm{1}$.
Therefore, we have the following relation:
\begin{align}
\textstyle
\{\sum_{l=1}^d k^{(l)}\bm{u}^{(l)}-c(\sum_{l=1}^d k^{(l)}\bm{u}^{(l)})\bm{1}|k^{(l)}\in \bm{R}\}\subseteq\{\sum_{l=1}^d k^{(l)}\bm{u}^{(l)}+k^{(d+1)}\bm{1}|k^{(l)}\in \bm{R}\},
\label{hougan}
\end{align}
where $\{\sum_{l=1}^d k^{(l)}\bm{u}^{(l)}+k^{(d+1)}\bm{1}|k^{(l)}\in \bm{R}\}$ is the vector space spanned by $\bm{u}^{(l)}$ and $\bm{1}$.
Let $Y$ be the vector space $\{\sum_{l=1}^d k^{(l)}\bm{u}^{(l)}+k^{(d+1)}\bm{1}|k^{(l)}\in \bm{R}\}$, the dimension of $Y$ becomes
\begin{align}
\textstyle
\mathrm{dim}(Y)=\begin{cases}
\textstyle d+1~~~\mathrm{if}~\bm{1}\notin \{\sum_{l=1}^d k^{(l)}\bm{u}^{(l)} |k^{(l)} \in \bm{R} \}, \\
\textstyle d~~~~~~~~~~\mathrm{if}~\bm{1}\in \{\sum_{l=1}^d k^{(l)}\bm{u}^{(l)} |k^{(l)}\in \bm{R} \} .
\end{cases}
\label{dimY}
\end{align}
We can see that $Y$ is the $d$ or $d+1$ dimensional linear subspace of $\bm{R}^M$.
From Eqs. (\ref{hougan}) and (\ref{dimY}), output vectors of log-softmax exist in the $d+1$ dimensional vector space, 
which completes the proof.
\end{proof}

Theorem \ref{sb2} shows that the log-softmax has at most $d+1$ linearly independent output vectors,
even if the various inputs are applied to the model.
Therefore, if the vectors of true log probabilities $\mathrm{log}P^*(\bm{y}|\bm{x})$
have more than $d+1$ linearly independent vectors, the softmax-based model cannot completely represent the true probabilities.
We can prove Theorem \ref{sb1} by using Theorem \ref{sb2} as follows:
\begin{proof}
If we have $d<\mathrm{rank}(\bm{A})-1$, i.e., $\mathrm{rank}(\bm{A})>d+1$, 
 the number of linearly independent vectors of $\mathrm{log}P^*(\bm{y}|\bm{x})$ is larger than $d+1$.
On the other hand, the output vectors $\mathrm{log}P_{\bm{\theta}}(\bm{y}|\bm{x})$
 of the model cannot be larger than $d+1$ linearly independent vectors from Theorem \ref{sb2}.
Therefore, the softmax-based model
cannot completely learn $P^*(\bm{y}|\bm{x})$, i.e., there exists a context $x$ in $\mathcal{L}$ such that $P_{\bm{\theta}}(Y|x)\neq P^{*} (Y|x)$. 
\end{proof}
The above analysis shows that the softmax bottleneck occurs because 
the output of log-softmax is the linear combination of the input $\bm{z}$ and vector $\bm{1}$ as \req{logsoft}.
Linear combination of the input and vector $\bm{1}$ increases the number of linearly independent vectors by at most one,
and as a result, the output vectors become at most $d+1$ linearly independent vectors.
The reason log-softmax becomes the linear combination is that the logarithm of the exponential function $\mathrm{log}(\mathrm{exp}(z))$ is $z$.

By contrast, the number of linearly independent output vectors of a nonlinear function
can be much greater than the number of linearly independent input vectors.
Therefore, if the other nonlinear functions are replaced with exponential functions,
the logarithm of such functions can be nonlinear and
the softmax bottleneck can be broken without additional parameters.

Our analysis provides new insights that the range of log-softmax is a subset of the less dimensional vector space
 although the dimension of a vector space is strongly related to the rank of a matrix.
Furthermore, our analysis explicitly shows the cause of the softmax bottleneck.
\subsection{Alternative functions to softmax and desirable properties}
\label{ppsec}
In the previous section, we explained that the softmax bottleneck can be broken by replacing nonlinear functions with exponential functions.
In this section, we explain the desirable properties of an alternative function to softmax.
We formulate a new output function $\bm{f}(\cdot)$ as follows:
\begin{align}
\textstyle
\left[\bm{f}(\bm{z})\right]_i=\frac{\left[\bm{g}(\bm{z})\right]_i}{\sum_{m=1}^M \left[\bm{g}(\bm{z})\right]_m}.
\label{altsoft}
\end{align}
The new function is composed of the nonlinear function $\bm{g}(\bm{z})$ and
the division for the normalization, so that the summation of the elements is one.
As the alternative function to softmax, 
a new output function $\bm{f}(\bm{z})$ and its $\bm{g}(\bm{z})$ should have all of the following properties:
\begin{description}
 \item[Nonlinearity of $\mathrm{log}(\bm{g}(\bm{z}))$]
 As mentioned in Secs.~\ref{sbsec1} and \ref{sbsec2}, softmax can be the bottleneck of the representational power because
 $\mathrm{log}(\mathrm{exp}(\bm{z}))$ is $\bm{z}$.
Provided that $\mathrm{log}(\bm{g}(\bm{z}))$ is a linear function, 
$\{\mathrm{log}(\bm{f}(\bm{z}))|\bm{z} \in S\}$ is a subset of the $d+1$ dimensional vector space.
In order to break the softmax bottleneck, $\log(\bm{g}(\bm{z}))$ should be nonlinear.
\item[Numerically stable]
In training of deep learning, we need to calculate the gradient for optimization.
The derivative of logarithm of $\left[\bm{f}(\bm{z})\right]_i $ with respect to $z_j$ is
\begin{align}
\textstyle
\frac{\partial\mathrm{log}(\left[\bm{f}(\bm{z})\right]_i)}{\partial z_j}=
\frac{1}{\left[\bm{f}(\bm{z})\right]_i}\frac{\partial \left[\bm{f}(\bm{z})\right]_i}{\partial z_j}.
\end{align}
We can see that this function has a division by $\left[\bm{f}(\bm{z})\right]_i$. It can cause a division by zero since
$\left[\bm{f}(\bm{z})\right]_i$ can be close to zero if networks completely go wrong in training.
The alternative functions should avoid a division by zero similar to softmax as shown in \req{GradSoft}.
\item[Non-negative]
In \req{altsoft},  
all elements of $\bm{g}(\bm{z})$ should be non-negative to limit output in $[0,1]$. 
Therefore, $\bm{g}(\bm{z})$ should be non-negative: $[\bm{g}(\bm{z})]_i\geq 0$.
Note that if $\bm{g}(\bm{z})$ is non-positive, $\bm{f}(\bm{z})$ also are limited to $[0,1]$. We only mention non-negative since non-positive function $\bm{g}(\bm{z})$ can easily be non-negative as $-\bm{g}(\bm{z})$.
\item[Monotonically increasing]
$\bm{g}(\bm{z})$ should be monotonically increasing so that $\bm{f}(\bm{z})$ becomes a smoothed version of the argmax function \cite{softmax,bishop}.
If $\bm{g}(\bm{z})$ is monotonically increasing, we can obtain the label that has the maximum value of $\bm{f}(\bm{z})$
by comparing elements of $\bm{z}$.
\end{description}

Note that, if we use ReLU 
 as $\bm{g}(\bm{z})$, 
the ReLU-based function $\bm{f}(\bm{z})$ does not have all the above properties since the gradient of its logarithm
is not numerically stable. 
If we use sigmoid as $\bm{g}(\bm{z})$, the new sigmoid-based function satisfies the above properties.
However, the output of sigmoid is bounded above as $[\bm{g}(\bm{z})]_i \leq 1$, and this restriction might limit the representational power.  
In fact, the sigmoid-based function does not outperform softmax on the large dataset
in  Sec.~\ref{experiment}.
We discuss these functions in detail in the appendix.
In the next section, we propose a new output activation function that can break the softmax bottleneck, and satisfies all the above properties. 
\subsection{Sigsoftmax}
For breaking the softmax bottleneck, we propose sigsoftmax given as follows:
\begin{definition}
Sigsoftmax is defined as
\begin{align}
\textstyle
\left[ \bm{f}(\bm{z})\right]_i&
\textstyle
=\frac{\mathrm{exp}({z_i})\sigma({z_i})}{\sum_{m=1}^M \mathrm{exp}({z_m})\sigma({z_m})},
\label{sseq}
\end{align}
where $\sigma(\cdot)$ represents a sigmoid function.
\label{defss}
\end{definition}
We theoretically show that sigsoftmax can break the softmax bottleneck and has the desired properties.
In the same way as in the analysis of softmax in Sec.~\ref{sbsec2}, we examine the range of log-sigsoftmax.
Since we have $\mathrm{log}(\sigma(z))\!=\!\mathrm{log}(\frac{1}{1+\mathrm{exp}(-z)})\!=\!z\!-\!\mathrm{log}(1\!+\!\mathrm{exp}(z))$, log-sigsoftmax becomes 
\begin{align}
\textstyle
\mathrm{log}(\bm{f}(\bm{z}))=2\bm{z}-\mathrm{log}(\bm{1}+\mathrm{exp}(\bm{z}))+c'(\bm{z})\bm{1},
\end{align}
where $c'(\bm{z})=\mathrm{log}(\sum_{m=1}^{M} \mathrm{exp}(z_m)\sigma(z_m))$, and
$\mathrm{log}(\bm{1}+\mathrm{exp}(\bm{z}))$ is the nonlinear function called softplus \cite{goodfellow2016deep}.
Since log-sigsoftmax is composed of a nonlinear function, 
its output vectors can be greater than $d+1$ linearly independent vectors.
Therefore, we have the following theorem:
\begin{theorem}
Let $S\subseteq \bm{R}^M$ be the $d$ dimensional vector space and $\bm{z}\in S$ be input of log-sigsoftmax,
some range of log-sigsoftmax $\{\mathrm{log}(\bm{f}(\bm{z}))|\bm{z} \in S\}$ is not a subset of a $d+1$ dimensional vector space.
\end{theorem}
The detailed proof of this theorem is given in the appendix.
Theorem \ref{brsb} shows that sigsoftmax can break the softmax bottleneck;
even if the vectors of the true log probabilities are more than $d+1$ linearly independent vectors, 
the sigsoftmax-based model can learn the true probabilities.

However, the representational powers of sigsoftmax and softmax are difficult to compare only by using the theorem based on the vector space.
This is because both functions are nonlinear and their ranges are not necessarily vector spaces,
even though they are subsets of vector spaces.
Therefore, we directly compare the ranges of sigsoftmax and softmax as the following theorem:
\begin{theorem}
\label{rangetheorem}
Let $\bm{z}\in S$ be the input of sigsoftmax $\bm{f}(\cdot)$ and softmax $\bm{f}_s(\cdot)$. If the $S$ is a
$d$ dimensional vector space and $\bm{1}\in S$,
the range of softmax is a subset of the range of sigsoftmax 
\begin{align}
\textstyle
\{\bm{f}_s(\bm{z})|\bm{z}\in S\}\subseteq\{\bm{f}(\bm{z})|\bm{z}\in S\}.
\end{align}
\end{theorem}
\begin{proof}
If we have $\bm{1}\in S$, $S$ can be 
 written as $S=\{\sum_{l=1}^{d-1}k'^{(l)}\bm{u}'^{(l)}+k'^{(d)}\bm{1}|k'^{(l)}\in R\}$ where $\bm{u}'^{(l)}$ ($l=1,\dots,d-1$) and $\bm{1}$ are linearly independent vectors. 
In addition, the arbitrary elements of $S$ can be written as $\sum_{l=1}^{d-1}k'^{(l)}\bm{u}'^{(l)}+k'^{(d)}\bm{1}$, and thus, $\bm{z}=\sum_{l=1}^{d-1}k'^{(l)}\bm{u}'^{(l)}+k'^{(d)}\bm{1}$. 
For the output of softmax, by substituting $\bm{z}=\sum_{l=1}^{d-1}k'^{(l)}\bm{u}'^{(l)}+k'^{(d)}\bm{1}$ for \req{sbeq}, we have  
\begin{align}
\textstyle
\left[\bm{f}_s(\bm{z})\right]_i=\frac{\mathrm{exp}(\left[\sum_{l=1}^{d-1}k'^{(l)}\bm{u}'^{(l)}\right]_i+k'^{(d)})}{\sum_{m=1}^M \mathrm{exp}(\left[\sum_{l=1}^{d-1}k'^{(l)}\bm{u}'^{(l)}\right]_m+k'^{(d)})}
&\textstyle =\frac{\mathrm{exp}(\left[\sum_{l=1}^{d-1}k'^{(l)}\bm{u}'^{(l)}\right]_i)}{\sum_{m=1}^M \mathrm{exp}(\left[\sum_{l=1}^{d-1}k'^{(l)}\bm{u}'^{(l)}\right]_m)}.
\label{softout1}
\end{align}
As a result, the range of softmax becomes as follows:
\begin{align}
\textstyle
\left\{\bm{f}_s(\sum_{l=1}^{d-1}k'^{(l)}\bm{u}'^{(l)}+k'^{(d)}\bm{1})|k'^{(l)}\in \bm{R}\right\}
&\textstyle =\left\{\frac{\mathrm{exp}(\left[\sum_{l=1}^{d-1}k'^{(l)}\bm{u}'^{(l)}\right]_i)}{\sum_{m=1}^M \mathrm{exp}(\left[\sum_{l=1}^{d-1}k'^{(l)}\bm{u}'^{(l)}\right]_m)}
| k'^{(l)}\in \bm{R}
\right\}.
\label{softrange}
\end{align}
On the other hand, by substituting $\bm{z}=\sum_{l=1}^{d-1}k'^{(l)}\bm{u}'^{(l)}+k'^{(d)}\bm{1}$ for \req{sseq}, output of sigsoftmax becomes as follows:
\begin{align}
\textstyle
\left[\bm{f}(\bm{z})\right]_i
&\textstyle
=\frac{\mathrm{exp}(\left[\sum_{l=1}^{d-1}k'^{(l)}\bm{u}'^{(l)}\right]_i)\sigma(\left[\sum_{l=1}^{d-1}k'^{(l)}\bm{u}'^{(l)}\right]_i+k'^{(d)})}
{\sum_{m=1}^M \mathrm{exp}(\left[\sum_{l=1}^{d-1}k'^{(l)}\bm{u}'^{(l)}\right]_m)\sigma(\left[\sum_{l=1}^{d-1}k'^{(l)}\bm{u}'^{(l)}\right]_m+k'^{(d)})}.
\label{sigsoftout1}
\end{align}
When $k'^{(l)}$ are fixed for $l=1,\dots, d-1$ and $k'^{(d)}\rightarrow+\infty$,\footnote{
Even though $k'^{(d)}$ is extremely large, 
the input vector is the element of the input space $S$.
}
 we have the following equality:
\begin{align}
\lim_{k'^{(d)}\!\rightarrow \!+\infty}\textstyle\!\frac{\mathrm{exp}(\left[\sum_{l=1}^{d-1}k'^{(l)}\bm{u}'^{(l)}\right]_i)\sigma(\left[\sum_{l=1}^{d-1}k'^{(l)}\bm{u}'^{(l)}\right]_i+k'^{(d)})}
{\sum_{m=1}^M
\mathrm{exp}(\left[\!\sum_{l=1}^{d-1}\!k'^{(l)}\bm{u}'^{(l)}\right]_m\!)\sigma(\left[\sum_{l=1}^{d-1}\!k'^{(l)}\bm{u}'^{(l)}\right]_m+k'^{(d)})}
\!=\!\frac{\mathrm{exp}(\left[\sum_{l=1}^{d-1}k'^{(l)}\bm{u}'^{(l)}\right]_i)}{\sum_{m=1}^M \mathrm{exp}
(\left[\sum_{l=1}^{d-1}k'^{(l)}\bm{u}'^{(l)}\right]_m)}\!,
\label{sigsoftout12}
\end{align}
since $\lim_{k\!\rightarrow\!+\infty}\!\sigma(v\!+\!k\!)\!=\!1$ when $v$ is fixed. 
From \req{sigsoftout12}, sigsoftmax has the following relation:
\begin{align}
\textstyle
\left\{
\frac{\mathrm{exp}(\left[\sum_{l=1}^{d-1}k'^{(l)}\bm{u}'^{(l)}\right]_i)}{\sum_{m=1}^M \mathrm{exp}
(\left[\sum_{l=1}^{d-1}k'^{(l)}\bm{u}'^{(l)}\right]_m)}|k'^{(l)}\in \bm{R}\right\}
=
\left\{
\bm{f}(\bm{z})|\bm{z}\in S'\right\}
\subseteq \left\{
\bm{f}(\bm{z})|\bm{z}\in S\right\},
\label{sigsofrange}
\end{align}
where $S'$ is a hyperplane of $S$ with $k'^{(d)}=+\infty$, $S'=\{\sum_{l=1}^{d-1}k'^{(l)}\bm{u}'^{(l)}+k'^{(d)}\bm{1}|k'^{(l)}\!\in\!\bm{R}~ \mathrm{for}~l=1,\dots,d-1 , k'^{(d)}=+\infty \}\subset S$.
From Eqs. (\ref{softrange}) and (\ref{sigsofrange}), we can see that the range of sigsoftmax includes the range of softmax.
Therefore, we have 
$\{\bm{f}_s(\bm{z})|\bm{z}\in S\}\subseteq\{\bm{f}(\bm{z})|\bm{z}\in S\}$.
\end{proof}
Theorem \ref{rangetheorem} shows that the range of sigsoftmax can be larger than that of softmax
if $\bm{1}\in S$.
The assumption $\bm{1}\in S$ means that there exist inputs of which outputs are the equal probabilities for all labels
 as $p_{\bm{\theta}}(y_i|\bm{x})=\frac{1}{M}$ for all $i$. This assumption is not very strong in practice. 
If $\bm{1}\notin S$, the range of sigsoftmax can include the range of softmax
by introducing one learnable scalar parameter $b$ into sigsoftmax as
 $\left[ \bm{f}(\bm{z}+b\bm{1})\right]_i=\frac{\mathrm{exp}({z_i})\sigma({z_i+b})}{\sum_{m=1}^M \mathrm{exp}({z_m})\sigma({z_m+b})}$.
In this case, if softmax can fit the true probability, $b$ can become large enough for sigsoftmax to approximately equal softmax.
In the experiments, we did not use $b$ in order to confirm that sigsoftmax can outperform softmax without additional parameters.
From Theorems \ref{brsb} and \ref{rangetheorem},
sigsoftmax can break the softmax bottleneck, and furthermore,
the representational power of sigsoftmax can be higher than that of softmax.

 Then, we show that sigsoftmax has the desirable properties introduced
in Sec.~\ref{ppsec} as shown in the following theorem 
from Definition \ref{defss} although we show its proof in the appendix:
\begin{theorem}
\label{SSPro}
Sigsoftmax has the following properties:
\begin{enumerate}
 \item Nonlinearity of $\mathrm{log}(\bm{g}(\bm{z}))$: $\mathrm{log}(\bm{g}(\bm{z}))=2\bm{z}-\mathrm{log}(\bm{1}+\mathrm{exp}(\bm{z}))$.
\item Numerically stable:
$\textstyle 
\frac{\partial \mathrm{log}\left[\bm{f}(\bm{z})\right]_i}{\partial z_j}=
\begin{cases}
\textstyle (1-\left[\bm{f}(\bm{z})\right]_j)(2-\sigma(z_j))~~~&i=j,\\
\textstyle -\left[\bm{f}(\bm{z})\right]_j(2-\sigma(z_j))~~~&i\neq j.
\end{cases}$
\item Non-negative: $[\bm{g}(\bm{z})]_i=\mathrm{exp}(z_i)\sigma(z_i)\geq 0$.
\item Monotonically increasing: $z_1\leq z_2 \Rightarrow \mathrm{exp}(z_1)\sigma(z_1)\leq \mathrm{exp}(z_2)\sigma(z_2)$.
\end{enumerate}
\end{theorem}
Since sigsoftmax is an alternative function to softmax, we can use the weighted sum of sigsoftmax functions
in the same way as MoS.
Mixture of sigsoftmax (MoSS) is the following function:
\begin{align}
\textstyle
P_{\bm{\theta}}(y_i|x)=\sum_{k=1}^{K}\pi(x,k)\frac{\mathrm{exp} (\left[\bm{W}\bm{h}(x,k)\right]_i)\sigma(\left[\bm{Wh}(x,k)\right]_i)}{\sum_{m=1}^{M}\mathrm{exp} (\left[\bm{W}\bm{h}(x,k)\right]_m)\sigma(\left[\bm{Wh}(x,k)\right]_m)}.
\end{align}
$\pi({x},k)$ is also composed of sigsoftmax as 
$\pi({x},k)=\frac{\mathrm{exp}(\bm{w}_{\pi,k}^T\bm{h}'(x))\sigma(\bm{w}_{\pi,k}^T\bm{h}'(x))}{\sum_{k'=1}^K\mathrm{exp}(\bm{w}_{\pi, k'}^T\bm{h}'(x))\sigma(\bm{w}_{\pi,k'}^T\bm{h}'(x))}$.
\section{Experiments}
\label{experiment}
\subsection{Experimental conditions}
To evaluate the effectiveness of sigsoftmax, we conducted experiments on language modeling.
We compared sigsoftmax with softmax, the ReLU-based function and the sigmoid-based function.
We also compared the mixture of sigsoftmax with that of softmax; MoSS with MoS.
We used Penn Treebank dataset (PTB) \cite{marcus1993building,tomas} and WikiText-2 dataset (WT2) \cite{WT2_2017}
by following the previous studies \cite{merity2018regularizing,krause2017dynamic,softmaxbottle}.
PTB is commonly used to evaluate the performance of RNN-based language modeling \cite{tomas,zaremba2014recurrent,merity2018regularizing,softmaxbottle}.
PTB is split into a training set (about 930~k tokens), validation set (about 74~k tokens),
and test set (about 82~k tokens).
The vocabulary size $M$ was set to 10~k, and all words outside the vocabulary were replaced with a special token.
WT2 is a collection of tokens from the set of articles on Wikipedia.
WT2 is also split into a training set (about 2100~k), validation set (about 220~k), and test set (about 250~k).
The vocabulary size $M$ was 33,278.
Since WT2 is larger than PTB,
language modeling of WT2 may require more representational power than that of PTB.

We trained a three-layer long short-term memory (LSTM) model with each output function.
After we trained models, we finetuned them and applied the dynamic evaluation \cite{krause2017dynamic}.
For fair comparison, the experimental conditions, such as unit sizes, dropout rates, initialization, 
and the optimization method
were the same as in the previous studies \cite{merity2018regularizing,softmaxbottle, krause2017dynamic} except for the number of epochs
by using their codes.\footnote{\url{https://github.com/salesforce/awd-lstm-lm} (Note that \citet{merity2018regularizing} further tuned some hyper-parameters to obtain results better than those in the original paper in their code.); \\ \url{https://github.com/benkrause/dynamic-evaluation}; \url{https://github.com/zihangdai/mos}  }
We set the epochs to be twice as large as the original epochs used in \cite{merity2018regularizing}
since the losses did not converge in the original epochs.
In addition, we trained each model with various random seeds
 and evaluated the average and standard deviation of validation and test perplexities for each method. 
The detailed conditions and the results at training and finetuning steps are provided in the appendix.
\subsection{Experimental results}
Validation perplexities and test perplexities of PTB and WT2 modeling are listed in Tabs.~\ref{PTB results} and \ref{WT2 results}.
\rtab{PTB results} shows that the sigmoid-based function achieved the lowest perplexities among output activation functions on PTB.
However, the sigmoid-based function did not outperform softmax on WT2.
This is because sigmoid is bounded above by one, $\sigma(\cdot)\leq 1$, and it may restrict the representational power.
As a result, the sigmoid based function did not perform well on the large dataset.
On the other hand, sigsoftmax achieved lower perplexities than softmax on PTB
 and achieved the lowest perplexities on WT2.
 Furthermore, between mixture models, MoSS achieved lower perplexities than MoS.
 Even though we trained and finetuned models under the conditions
that are highly optimized for softmax and MoS in \cite{merity2018regularizing, softmaxbottle},
sigsoftmax and MoSS outperformed softmax and MoS, respectively. 
Therefore, we conclude that sigsoftmax outperforms softmax as an activation function.
\begin{table}[t]
\caption{Results of the language modeling experiment on PTB.}
\centering
\begin{tabular}{ccccccc}
\toprule
&{\small Softmax}&{\small $g$:ReLU}&{\small $g$: Sigmoid}&{\small Sigsoftmax}&{\small MoS}&{\small MoSS}\\ \cmidrule(r){1-5}\cmidrule(l){6-7}
{\small Validation}&{\small 51.2$\pm$0.5}&{\small (4.91$\pm$5)$\times 10^{3}$}&{\small \bf{49.2}$\pm$0.4 } &{\small 49.7$\pm$0.5 }&{\small 48.6$\pm$0.2}&{\small \bf{48.3}$\pm$0.1}\\
{\small Test}&{\small 50.5$\pm$0.5}&{\small (2.78$\pm$8)$\times 10^5$} &{\small \bf{48.9}$\pm$0.3}  &{\small 49.2$\pm$0.4} &{\small 48.0$\pm$0.1}&{\small \bf{47.7}$\pm$0.07} \\
\bottomrule
\end{tabular}
\label{PTB results}
\newline
\newline
\caption{Results of the language modeling experiment on WT2.}
\centering
\begin{tabular}{ccccccc}
\toprule
&{\small Softmax}&{\small $g$:ReLU}&{\small $g$:Sigmoid}&{\small Sigsoftmax}&{\small MoS}&{\small MoSS}\\ \cmidrule(r){1-5}\cmidrule(l){6-7}
{\small Validation}&{\small 45.3$\pm$0.2}&{\small (1.79$\pm$0.8)$\times 10^{3}$}&{\small 45.7$\pm$0.1 }  &{\small \bf{44.9$\pm$0.1}}&{\small 42.5$\pm$0.1}&{\small \bf{42.1$\pm$0.2}} \\
{\small Test}&{\small 43.3$\pm$0.1}&{\small (2.30$\pm$2)$\times 10^{4}$}&{\small 43.5$\pm$0.1  } &{\small \bf{42.9$\pm$0.1}}&{\small 40.8$\pm$0.03}&{\small \bf{40.3$\pm$0.2} }\\
\bottomrule
\end{tabular}
\label{WT2 results}
\newline
\newline
\caption{The number of linearly independent log-output vectors on test datasets: Ranks of $\hat{\bm{A}}$.}
\centering
\begin{tabular}{ccccccc}
\toprule
&{\small Softmax}&{\small $g$: ReLU}&{\small $g$: Sigmoid}&{\small Sigsoftmax}&{\small MoS}&{\small MoSS}\\ \cmidrule(r){1-5}\cmidrule(l){6-7}
{\small PTB}&{\small 402}&{\small 8243}&{\small 1304}  &{\small 4640}&{\small 9980}&{\small 9986} \\
{\small WT2}&{\small 402}&{\small 31400}&{\small 463 }&{\small 5465}&{\small 12093}&{\small 19834} \\
\bottomrule
\end{tabular}
\label{Rank Table}
\end{table}
\subsection{Evaluation of linear independence}
In this section, we evaluate linear independence of output vectors of each function.
First, we applied whole test data to the finetuned models and obtained 
log-output $\mathrm{log}(P_{\bm{\theta}}(\bm{y}_t |\bm{x}_t))$, e.g., log-softmax, at each time.
Next, we made the matrices $\hat{\bm{A}}$ as 
$\hat{\bm{A}}=[\mathrm{log}(P_{\bm{\theta}}(\bm{y}_1 |\bm{x}_1)),\dots,\mathrm{log}(P_{\bm{\theta}}(\bm{y}_T |\bm{x}_T))]\in \bm{R}^{M\times T}$
where $T$ is the number of tokens of test data. 
$M$ and $T$ were respectively 10,000 and 82,430 on the PTB test set
and 33,278 and 245,570 on the WT2 test set. 
Finally, we examined the rank of $\hat{\bm{A}}$ since the rank of the matrix is $N$
if the matrix is composed of $N$ linearly independent vectors.
Note that the numerical approaches for computing ranks have roundoff error,
 and we used the threshold used in \cite{press2007numerical,softmaxbottle} to detect the ranks.
The ranks of $\hat{\bm{A}}$ are listed in \rtab{Rank Table}.
The calculated singular values for detecting ranks are presented in the appendix.

We can see that log-softmax output vectors have 402 linearly independent vectors.
In the experiments, the number of hidden units is set to 400,
and we used a bias vector in the output layer. As a result, the dimension of the input space $S$ was at most 401, and
log-softmax output vectors are theoretically at most 402 linearly independent vectors from Theorem \ref{sb2}.
Therefore, we confirmed that the range of log-softmax is a subset of the $d+1$ dimensional vector space.
On the other hand, the number of linearly independent output vectors of sigsoftmax, ReLU and sigmoid-based functions are not bounded by 402.
Therefore, sigsoftmax, ReLU and sigmoid-based functions can break the softmax bottleneck.
The ranks of the ReLU-based function are larger than the other activation functions.
However, the ReLU-based function is numerically unstable as mentioned in Sec.~\ref{ppsec}.
As a result, it was not trained well as shown in Tabs.~\ref{PTB results} and \ref{WT2 results}.
MoSS has more linearly independent output vectors than MoS.
Therefore, MoSS may have more representational power than MoS.
\section{Conclusion}
In this paper, we investigated the range of log-softmax 
and identified the cause of the softmax bottleneck.
We proposed sigsoftmax, which can break the softmax bottleneck
 and has more representational power than softmax without additional parameters.
Experiments on language modeling demonstrated that sigsoftmax outperformed softmax.
Since sigsoftmax has the desirable properties for output activation functions,
it has the potential to replace softmax in many applications.
\bibliographystyle{plainnat}
 \bibliography{bif.bib}
 \appendix
\section*{Appendix}
\section{Proofs of theorems}
In this section, we provide the proofs of theorems that are not provided in the paper.
\setcounter{theorem}{2}
\begin{theorem}
Let $S\subseteq \bm{R}^M$ be the $d$ dimensional vector space and $\bm{z}\in S$ be input of log-sigsoftmax,
some range of log-sigsoftmax $\{\mathrm{log}(\bm{f}(\bm{z}))|\bm{z} \in S\}$ is not a subset of a $d+1$ dimensional vector space.\label{brsb}
\end{theorem}
\begin{proof}
We prove this by contradiction.
If Theorem~\ref{brsb} does not hold, every range of log-sigsoftmax 
$\{\mathrm{log}(\bm{f}(\bm{z}))|\bm{z}\in S\}$ is a subset of a $d+1$ dimensional
 vector space.
When we provide a counterexample of this statement,
we prove Theorem~\ref{brsb} since this statement is the negation of Theorem~\ref{brsb}.
The counter example is the case in which $S$ is the one dimensional vector space (i.e., $d=1$), 
$S=\{k\bm{u}|k\in R\}$ and $\bm{u}=[1,2,0]^T$.
Under the above condition, from Definition 1 in the paper,
outputs of log-sigsoftmax are as follows:
\begin{align}
\textstyle \mathrm{log}(\bm{f}(\bm{z}))=\begin{bmatrix}
2k\\4k\\ 0\end{bmatrix}-
\begin{bmatrix}
\mathrm{log}(1+\mathrm{exp}(k))\\
\mathrm{log}(1+\mathrm{exp}(2k))\\
 \mathrm{log}(2)\end{bmatrix}
-\mathrm{log}\left(\frac{1+2\sigma(k)\mathrm{exp}(k)+2\sigma(2k)\mathrm{exp}(2k)}{2} \right)\bm{1}.
\end{align}
From $S=\{k\bm{u}|k\in R\}$, we choose three inputs $\bm{z}_1=[0,0,0]^T, \bm{z}_2=[1,2,0]^T, \bm{z}_3=[-1,-2,0]^T$
and investigate the outputs.
The outputs of log-sigsoftmax are as follows:
\begin{align}
\textstyle\mathrm{log}(\bm{f}(\bm{z}_1))&\textstyle =-\mathrm{log}(3)\bm{1},\\
\textstyle\mathrm{log}(\bm{f}(\bm{z}_2))&\textstyle =\begin{bmatrix}
2-\mathrm{log}(1+\mathrm{exp}(1))
\\4-\mathrm{log}(1+\mathrm{exp}(2))\\
- \mathrm{log}(2)\end{bmatrix}
-\mathrm{log}\left(\frac{1+2\sigma(1)\mathrm{exp}(1)+2\sigma(2)\mathrm{exp}(2)}{2} \right)\bm{1},\\
\textstyle \mathrm{log}(\bm{f}(\bm{z}_3))&\textstyle =\begin{bmatrix}
-2-\mathrm{log}(1+\mathrm{exp}(-1))
\\-4-\mathrm{log}(1+\mathrm{exp}(-2))\\
- \mathrm{log}(2)\end{bmatrix}
-\mathrm{log}\left(\frac{1+2\sigma(-1)\mathrm{exp}(-1)+2\sigma(-2)\mathrm{exp}(-2)}{2} \right)\bm{1}.
\end{align}
To evaluate linear independence, we examine the solution of the $\alpha_1\mathrm{log}(\bm{f}(\bm{z}_1))+
\alpha_2\mathrm{log}(\bm{f}(\bm{z}_2)+
\alpha_3\mathrm{log}(\bm{f}(\bm{z}_3)=\bm{0}$.
If its solution is only $\alpha_1=\alpha_2=\alpha_3=0$, $\mathrm{log}(\bm{f}(\bm{z}_1))$, $\mathrm{log}(\bm{f}(\bm{z}_2)$, and
$\mathrm{log}(\bm{f}(\bm{z}_3)$ are linearly independent.
Each element of 
$\alpha_1\mathrm{log}(\bm{f}(\bm{z}_1))+
\alpha_2\mathrm{log}(\bm{f}(\bm{z}_2)+
\alpha_3\mathrm{log}(\bm{f}(\bm{z}_3)=\bm{0}$ becomes the following equations:
\begin{numcases}
{}\textstyle
-\alpha_1\mathrm{log}(3)+\alpha_2\{2-\mathrm{log}(1+\mathrm{exp}(1))-\mathrm{log}(\frac{1+2\sigma(1)\mathrm{exp}(1)+2\sigma(2)\mathrm{exp}(2))}{2}\} &\nonumber\\ \textstyle
 +\alpha_3\{-2-\mathrm{log}(1+\mathrm{exp}(-1))-\mathrm{log}(\frac{1+2\sigma(-1)\mathrm{exp}(-1)+2\sigma(-2)\mathrm{exp}(-2))}{2}\}=0,&
\label{eq1}\\ \textstyle
-\alpha_1\mathrm{log}(3)+\alpha_2\{4-\mathrm{log}(1+\mathrm{exp}(2))-\mathrm{log}(\frac{1+2\sigma(1)\mathrm{exp}(1)+2\sigma(2)\mathrm{exp}(2))}{2}\} &\nonumber\\ \textstyle
 +\alpha_3\{-4-\mathrm{log}(1+\mathrm{exp}(-2))-\mathrm{log}(\frac{1+2\sigma(-1)\mathrm{exp}(-1)+2\sigma(-2)\mathrm{exp}(-2))}{2}\}=0,&
\label{eq2}\\ \textstyle
-\alpha_1\mathrm{log}(3)+\alpha_2\{-\mathrm{log}(2)-\mathrm{log}(\frac{1+2\sigma(1)\mathrm{exp}(1)+2\sigma(2)\mathrm{exp}(2))}{2}\}&\nonumber\\ \textstyle
 +\alpha_3\{-\mathrm{log}(2)-\mathrm{log}(\frac{1+2\sigma(-1)\mathrm{exp}(-1)+2\sigma(-2)\mathrm{exp}(-2))}{2}\}=0.&\label{eq3}
\end{numcases}
From \req{eq3}, we have
\begin{align}
\textstyle
\alpha_1=&\frac{\!\alpha_2}{\mathrm{log}(3)}\{-\mathrm{log}(2)-\mathrm{log}(\frac{1+2\sigma(1)\mathrm{exp}(1)+2\sigma(2)\mathrm{exp}(2))}{2}\}\nonumber\\
&+\frac{\alpha_3}{\mathrm{log}(3)}\{-\mathrm{log}(2)-\mathrm{log}(\frac{1+2\sigma(-1)\mathrm{exp}(-1)+2\sigma(-2)\mathrm{exp}(-2))}{2}\}.
 \label{al1}
\end{align}
Substituting \req{al1} for \req{eq1} and \req{eq2}, we have
\begin{numcases}
{}\textstyle
\alpha_2\{2-\mathrm{log}(1+\mathrm{exp}(1))+\mathrm{log}(2)\} 
 +\alpha_3\{-2-\mathrm{log}(1+\mathrm{exp}(-1))+\mathrm{log}(2)\}=0,&
\label{eq4}\\ \textstyle
\alpha_2\{4-\mathrm{log}(1+\mathrm{exp}(2))+\mathrm{log}(2)\} 
 +\alpha_3\{-4-\mathrm{log}(1+\mathrm{exp}(-2))+\mathrm{log}(2)\}=0.&
\label{eq5}
\end{numcases}
From \req{eq4}, we have
\begin{align}
\textstyle
\alpha_2=\alpha_3\frac{\{2+\mathrm{log}(1+\mathrm{exp}(-1))-\mathrm{log}(2)\}}{\{2-\mathrm{log}(1+\mathrm{exp}(1))+\mathrm{log}(2)\} },
 \label{al2}
\end{align}
and substituting \req{al2} for \req{eq5}, we have
\begin{align}
\textstyle
\alpha_3\left[ \frac{\{2+\mathrm{log}(1+\mathrm{exp}(-1))-\mathrm{log}(2)\}\{4-\mathrm{log}(1+\mathrm{exp}(2))+\mathrm{log}(2)\} 
}{\{2-\mathrm{log}(1+\mathrm{exp}(1))+\mathrm{log}(2)\} }
\!\!+\!\!\{-4-\mathrm{log}(1+\mathrm{exp}(-2))+\mathrm{log}(2)\}\right]
\!=\!0.&
\label{eq6}
\end{align}
The solution of \req{eq6} is only $\alpha_3=0$, and thus, $\alpha_1=\alpha_2=\alpha_3=0$.
We have
$\alpha_1\mathrm{log}(\bm{f}(\bm{z}_1))+
\alpha_2\mathrm{log}(\bm{f}(\bm{z}_2)+
\alpha_3\mathrm{log}(\bm{f}(\bm{z}_3)=\bm{0}$ 
if and only if $\alpha_1=\alpha_2=\alpha_3=0$.
Therefore, $\mathrm{log}(\bm{f}(\bm{z}_1))$, $\mathrm{log}(\bm{f}(\bm{z}_2))$ and $\mathrm{log}(\bm{f}(\bm{z}_3))$
are linearly independent, i.e., output vectors can be three linearly independent vectors
even though $d+1=2$. 
Therefore, the output of log-sigsoftmax can be greater than $d+1$ linearly independent vectors, and thus,
the range of log-sigsoftmax is not a subset of $d+1$ dimension vector space.
This contradicts the statement that every range of log-sigsoftmax 
$\{\mathrm{log}(\bm{f}(\bm{z}))|\bm{z}\in S\}$ is a subset of a $d+1$ dimensional
 vector space.
 As a result, some range of log-sigsoftmax is not a subset of a $d+1$ dimensional vector space.
\end{proof}
\setcounter{theorem}{4}
\begin{theorem}
Sigsoftmax has the following properties:
\begin{enumerate}
 \setlength{\itemsep}{-1pt}
 \item Nonlinearity of $\mathrm{log}(\bm{g}(\bm{z}))$: $\mathrm{log}(\bm{g}(\bm{z}))=2\bm{z}-\mathrm{log}(\bm{1}+\mathrm{exp}(\bm{z}))$.
\item Numerically stable:
$\textstyle 
\frac{\partial \mathrm{log}\left[\bm{f}(\bm{z})\right]_i}{\partial z_j}=
\begin{cases}
\textstyle (1-\left[\bm{f}(\bm{z})\right]_j)(2-\sigma(z_j))~~~&i=j,\\
\textstyle -\left[\bm{f}(\bm{z})\right]_j(2-\sigma(z_j))~~~&i\neq j.
\end{cases}$
\item Non-negative: $[\bm{g}(\bm{z})]_i=\mathrm{exp}(z_i)\sigma(z_i)\geq 0$.
\item Monotonically increasing: $z_1\leq z_2 \Rightarrow \mathrm{exp}(z_1)\sigma(z_1)\leq \mathrm{exp}(z_2)\sigma(z_2)$.
\end{enumerate}
\end{theorem}
\begin{proof}
First, we have $\mathrm{log}(\bm{g}(\bm{z}))=2\bm{z}-\mathrm{log}(\bm{1}+\mathrm{exp}(\bm{z}))$
since $[\bm{g}(\bm{z})]_i=\mathrm{exp}(z_i)\sigma(z_i)=\frac{\mathrm{exp}(z_i)}{1+\mathrm{exp}(-z_i)}=\frac{\mathrm{exp}(2z_i)}{1+\mathrm{exp}(z_i)}$.
 $\mathrm{log}(1+\mathrm{exp}(z))$ is softplus and is a nonlinear function.
Therefore, $\mathrm{log}(\bm{g}(\bm{z}))$ is a nonlinear function.
Second, since we have $\frac{\mathrm{d} \mathrm{exp}(z)}{\mathrm{d} z}=\mathrm{exp}(z)$ and  $\frac{\mathrm{d} \sigma(z)}{\mathrm{d} z}
=\sigma(z)(1-\sigma(z))$, 
we have \begin{align*}
\textstyle 
\frac{\partial \mathrm{log}\left[\bm{f}(\bm{z})\right]_i}{\partial z_j}
&\textstyle=\!\frac{1}{\left[\bm{f}(\bm{z})\right]_i}\frac{\partial \left[\bm{f}(\bm{z})\right]_i}{\partial z_j},\\
& \textstyle=\!\frac{1}{\left[\bm{f}(\bm{z})\right]_i}\!\!
\left\{\!\frac{1}{\sum_{m=1}^M \mathrm{exp}(z_m)\sigma(z_m)}\frac{\partial\mathrm{exp}(z_i)\sigma(z_i)}{\partial z_j}
\!-\!\frac{\left[\bm{f}(\bm{z})\right] _i}{\sum_{m=1}^M\! \mathrm{exp}(z_m)\sigma(z_m)}\frac{\partial \sum_{m=1}^M\! \mathrm{exp}(z_m)\sigma(z_m)}{\partial z_j}\!\right\},\\
&\textstyle=\!
\begin{cases}
\textstyle (1-\left[\bm{f}(\bm{z})\right]_j)(2-\sigma(z_j))~~~&i=j,\\
\textstyle -\left[\bm{f}(\bm{z})\right]_j(2-\sigma(z_j))~~~&i\neq j.
\end{cases}.
\end{align*}
Third, since $\mathrm{\exp}(z)\geq 0$ and $\sigma(z)\geq 0$, we have $\mathrm{\exp}(z)\sigma(z)\geq 0$.
Finally, the derivative of $\mathrm{exp}(z)\sigma(z)$ is 
 $\frac{\mathrm{d} \mathrm{exp}(z)\sigma(z)}{\mathrm{d} z}=\frac{\mathrm{exp}(z)+2}{(1+\mathrm{exp}(-z))^2}\geq 0$ for all $z$,
  and thus, $\mathrm{exp}(z)\sigma(z)$ is monotonically increasing.
\end{proof}
\section{Properties of output functions composed of ReLU and sigmoid}
In this section, we investigate properties of output functions composed of ReLU and sigmoid.
\subsection{ReLU-based output function}
A ReLU-based output function is given as
\begin{align}
\textstyle 
\left[\bm{f}(\bm{z})\right]_i =\frac{\mathrm{ReLU}(z_i)}{\sum_{m=1}^M\mathrm{ReLU}(z_m)}.
\end{align}

The ReLU-based function does not satisfy all the desirable properties as follows:
\begin{enumerate}
 \setlength{\itemsep}{-1pt}
\item Nonlinearity of $\mathrm{log}(\bm{g}(\bm{z}))$:
The logarithm of ReLU is as follows:
\begin{align*}
[\mathrm{log}(\mathrm{ReLU}(\bm{z}))]_i=\begin{cases}
\mathrm{log}(z_i)~~ &\mathrm{if} ~~z_i>0,\\
-\infty~~&\mathrm{if}~~z_i\leq 0.
\end{cases}
\end{align*}
This function is obviously nonlinear.
\item Numerically unstable:
\begin{align}
\textstyle 
\frac{\partial \mathrm{log}\left[\bm{f}(\bm{z})\right]_i}{\partial z_j}=
\begin{cases}
\textstyle  \frac{1}{\mathrm{ReLU}(z_i)}\frac{\partial \mathrm{ReLU}(z_i)}{\partial z_j} -\frac{1}{\sum_{m=1}^M\mathrm{ReLU}(z_m)}\frac{\partial \mathrm{ReLU}(z_i)}{\partial z_j} ~~~&i=j,\\
\textstyle -\frac{1}{\sum_{m=1}^M\mathrm{ReLU}(z_m)}\frac{\partial \mathrm{ReLU}(z_j)}{\partial z_j} ~~~&i\neq j.
\end{cases}\end{align}
We can see that the derivative of a ReLU-based function has the division by $\mathrm{ReLU}(z_i)$.
Since $\mathrm{ReLU}(z_i)$ can be close to zero, the calculation of gradient is numerically unstable.
\item Non-negative: $[\bm{g}(\bm{z})]_i=\max(z_i,0)$ is obviously greater than or equal to 0.
\item Monotonically increasing: $z_1\leq z_2 \Rightarrow\max(z_1,0)\leq \max(z_2,0)$ since the derivative of ReLU is always greater than or equal to 0.
\end{enumerate}

From the above, the ReLU-based function is numerically unstable.
Therefore, in the experiment, we use the following function:
\begin{align}
\textstyle \left[\bm{f}(\bm{z})\right]_i=\frac{\mathrm{ReLU}(z_i)+\varepsilon}{\sum_{m=1}^M\mathrm{ReLU}(z_m)+\varepsilon},
\label{ReLU-based}
\end{align}
where $\varepsilon$ is the hyper parameter of small value.
In the experiment, we used $\varepsilon=10^{-8}$.
\subsection{Sigmoid-based output function}
A sigmoid-based output function is given as
\begin{align}
\textstyle 
\left[\bm{f}(\bm{z})\right]_i =\frac{\sigma(z_i)}{\sum_{m=1}^M\sigma(z_m)}.
\end{align}

The sigmoid-based function satisfies the desirable properties as follows:
\begin{enumerate}
 \setlength{\itemsep}{-1pt}
\item Nonlinearity of $\mathrm{log}(\bm{g}(\bm{z}))$:
The logarithm of sigmoid is as follows:
\begin{align*}
\mathrm{log}(\sigma(\bm{z}))=\bm{z}-\mathrm{log}(\bm{1}+\mathrm{exp}(\bm{z})).
\end{align*}
This function is obviously nonlinear.
\item Numerically stable:
\begin{align*}
\textstyle 
\frac{\partial \mathrm{log}\left[\bm{f}(\bm{z})\right]_i}{\partial z_j}=
\begin{cases}
\textstyle (1-\left[\bm{f}(\bm{z})\right]_j)(1-\sigma(z_j))~~~&\textstyle i=j,\\
\textstyle -\left[\bm{f}(\bm{z})\right]_j(1-\sigma(z_j))~~~&\textstyle i\neq j.
\end{cases}
\end{align*}
We can see that this function does not have the division.
Therefore, the calculation of gradient is numerically stable.
\item Non-negative: We have $[\bm{g}(\bm{z})]_i=\sigma(z_i)\geq 0$.
However, sigmoid is also bounded by 1.
This may be the cause of the limitation of representational capacity.
\item Monotonically increasing: $z_1\leq z_2 \Rightarrow\sigma(z_1)\leq \sigma(z_2)$ since the derivative of sigmoid
$\sigma(z)(1-\sigma(z))$ is greater than or equal to 0.
\end{enumerate}
\section{Detailed experimental conditions and results}
\subsection{Experimental conditions}
\subsubsection{Conditions for activation functions}
For comparing softmax, sigsoftmax, the ReLU-based function, and the sigmoid-based function,
we trained a three-layer LSTM by following \cite{merity2018regularizing}.
We used the codes provided by \cite{merity2018regularizing,krause2017dynamic}.\footnote{\url{https://github.com/salesforce/awd-lstm-lm};
\\ \url{https://github.com/benkrause/dynamic-evaluation}}
 \citet{merity2018regularizing} further tuned
 some hyper parameters to obtain results better than those in the original paper \cite{merity2018regularizing} in their code.
For fair comparison, we only changed the code of \cite{merity2018regularizing} as
 (i) replacing softmax with sigsoftmax, the ReLU-based function and sigmoid-based function,
 (ii) using various random seeds, and (iii) using the epochs twice as large as the original epochs in \cite{merity2018regularizing}.
The ReLU-based function is defined by \req{ReLU-based}, and we used $\varepsilon=10^{-8}$.

We used the experimental conditions of \cite{merity2018regularizing}.
The number of units of LSTM layers was set to 1150, and embedding size was 400.
Weight matrices were initialized with a uniform distribution $U(-0.1,0.1)$ for the embedding layer,
and all other weights were initialized with $U(-\frac{1}{\sqrt{H}},\frac{1}{\sqrt{H}})$
where $H$ is the number of hidden units.

All models were trained by a non-monotonically triggered variant of averaged SGD (NT-ASGD) \cite{merity2018regularizing}
with the learning rate of 30, and we carried out gradient clipping with the threshold of 0.25.
We used dropout connect \cite{merity2018regularizing}, and dropout rates on the word vectors,
on the output between LSTM layers, on the output of the final LSTM layer, and on the embedding layer
were set to (0.4,0.25,0.4,0.1) on PTB, and (0.65,0.2,0.4,0.1) on WT2.
Batch sizes were set to 20 on PTB, and 80 on WT2.
Numbers of training epochs were set to 1000 on PTB and 1500 on WT2.
After training, we ran ASGD as a fine-tuning step until the stopping criterion was met.
We used a random backpropagation through time (BPTT) length which is $\mathcal{N}(70,5)$ with probability 0.95
and $\mathcal{N}(35,5)$ with probability 0.05.
We applied activation regularization (AR) and temporal activation regularization (TAR)
to the output of the final RNN layer.
Their scaling coefficients were 2 and 1, respectively.

After the fine-tuning step, we used dynamic evaluation \cite{krause2017dynamic}.
In this step, we used grid-search for hyper parameter tuning provided by \cite{krause2017dynamic}.
The learning rate $\eta$ was tuned in  $[3\times 10^{-5}, 4\times 10^{-5}, 5\times 10^{-5}, 6\times 10^{-5}, 7\times 10^{-5},1\times 10^{-4}]$, and
decay rate $\lambda$ was tuned in  $[1\times 10^{-3}, 2\times 10^{-3}, 3\times 10^{-3}, 5\times 10^{-3}]$.
$\epsilon$ was set to 0.001, and batch size was set to 100.

 We applied the above procedure (training and finetuning each model, applying dynamic evaluation) 10 times,
  and evaluated the average of minimum validation perplexities
 and the average of test perplexities.
\subsubsection{Conditions for mixture models; MoS with MoSS}
We trained a three-layer LSTM by following \cite{softmaxbottle} to compare MoSS with MoS.
We also used the codes provided by \cite{softmaxbottle},\footnote{\url{https://github.com/zihangdai/mos}}
and only changed the code as (i) replacing MoSS with MoS, (ii) using various random seeds.
After we trained the models, we finetuned them and applied dynamic evaluation to the finetuned models.

We used the experimental conditions of \cite{softmaxbottle}. In this experiments, 
the numbers of units of three LSTM layers were set to [960, 960, 620], and embedding size was 280 on PTB.
On WT2, the numbers of units of LSTM layers were set to [1150, 1150, 650], and embedding size was 300.
The number of mixture was set to 15 on both datasets.
In the same way as the above experimental conditions,
weight matrices were initialized with a uniform distribution $U(-0.1,0.1)$ for the embedding layer,
and all other weights were initialized with $U(-\frac{1}{\sqrt{H}},\frac{1}{\sqrt{H}})$.
On both datasets, we used
word level variational drop out with the rate of 0.10,
recurrent weight dropout with rate of 0.5, and context vector level variational drop out with the rate of 0.30.
In addition, 
embedding level variational dropout with the rate of 0.55 and
hidden level variational dropout with the rate of 0.225
were used on PTB.
On WT2, we used
embedding level variational drop out with the rate of 0.40 and
hidden level variational drop out with the rate of 0.225.
The optimization method was the same as in the previous section.

At the dynamic evaluation step, 
the learning rate was set to 0.002 and batchsize was set to 100 on both datasets.
On PTB, $\epsilon$ was set to 0.001 and decay rate $\lambda$ was set to 0.075.
On WT2, we set $\epsilon$ to 0.002 and decay rate $\lambda$ to 0.02.
All the above conditions were the same as those in \cite{softmaxbottle}.

 We applied the above procedure five times and evaluated the average of minimum validation perplexities
 and the average of test perplexities.
\subsection{Results}
\begin{table}[tbp]
\caption{Results of the language modeling experiment on PTB. Valid. means the validation perplexity, and dynamic eval. means dynamic evaluation \cite{krause2017dynamic}.}
\centering
\scalebox{0.9}{
\begin{tabular}{ccccccc}
\toprule
&Softmax&$g$:ReLU&$g$: Sigmoid&Sigsoftmax&MoS&MoSS\\ \cmidrule(r){1-5}\cmidrule(l){6-7}
{\scriptsize Valid. w/o finetune}&61.1 $\pm$0.4&(1.85$\pm$0.2)$\times 10^3$&60.7 $\pm$0.2&61.0 $\pm$0.2&58.4$\pm$0.2&58.4$\pm$0.3 \\
{\scriptsize Test w/o finetune}&58.8 $\pm$0.4  &(1.54$\pm$0.2)$\times 10^3$&58.5 $\pm$0.2&58.4 $\pm$0.2&56.3$\pm$0.3&56.2$\pm$0.2 \\ \cmidrule(r){1-5}\cmidrule(l){6-7}
Valid.&59.2 $\pm$0.4&(1.51$\pm$0.1)$\times 10^3$&58.7 $\pm$0.4&59.2 $\pm$0.4&56.8$\pm$0.2&56.9$\pm$0.1 \\
Test&57.0 $\pm$0.6  &(1.24$\pm$0.08)$\times 10^3$&56.4 $\pm$0.2&56.6 $\pm$0.4& 54.7$\pm$0.08&54.6$\pm$0.2\\ \cmidrule(r){1-5}\cmidrule(l){6-7}
{\scriptsize Valid.+dynamic eval.}&51.2$\pm$0.5&(4.91$\pm$5)$\times 10^{3}$&\bf{49.2}$\pm$0.4  &49.7$\pm$0.5 &48.6$\pm$0.2&\bf{48.3}$\pm$0.1\\
{\scriptsize Test +dynamic eval.}&50.5$\pm$0.5&(2.78$\pm$8)$\times 10^5$ &\bf{48.9}$\pm$0.3  &49.2$\pm$0.4 &48.0$\pm$0.1&\bf{47.7}$\pm$0.07 \\
\bottomrule
\end{tabular}
}
\label{PTB results2}
\end{table}
\begin{table}[tbp]
\caption{Results of the language modeling experiment on WT2. Valid. means the validation perplexity, and dynamic eval. means dynamic evaluation \cite{krause2017dynamic}.}
\centering
\scalebox{0.9}{
\begin{tabular}{ccccccc}
\toprule
&Softmax&$g$:ReLU&$g$:Sigmoid&Sigsoftmax&MoS&MoSS\\ \cmidrule(r){1-5}\cmidrule(l){6-7}
{\scriptsize Valid. w/o finetune}&68.0$\pm$0.2 &(8.74$\pm$0.7)$\times 10^2$  &72.8$\pm 0.3$&67.8$\pm$0.1&65.9$\pm$0.5&65.1$\pm$0.2 \\
{\scriptsize Test w/o finetune}&65.2$\pm$0.2&(7.97$\pm$0.7)$\times 10^2$   &69.7$\pm$0.3&65.0$\pm$0.2&63.3$\pm$0.4&62.5$\pm$0.3 \\ \cmidrule(r){1-5}\cmidrule(l){6-7}
Valid.&67.4$\pm$0.2 &(6.48$\pm$0.1)$\times 10^2$  &70.8$\pm 0.1$&67.4$\pm$0.2&64.0$\pm$0.3&63.7$\pm$0.3 \\
Test&64.7$\pm$0.2&(5.93$\pm$0.08)$\times 10^2$   &68.2$\pm$0.1&64.2$\pm$0.1&61.4$\pm$0.4&61.1$\pm$0.3 \\ \cmidrule(r){1-5}\cmidrule(l){6-7}
{\scriptsize Valid.+dynamic eval.}&45.3$\pm$0.2&(1.79$\pm$0.8)$\times 10^{3}$&45.7$\pm$0.1   &\bf{44.9$\pm$0.1}&42.5$\pm$0.1&\bf{42.1$\pm$0.2} \\
{\scriptsize Test +dynamic eval.}&43.3$\pm$0.1&(2.30$\pm$2)$\times 10^{4}$&43.5$\pm$0.1   &\bf{42.9$\pm$0.1}&40.8$\pm$0.03&\bf{40.3$\pm$0.2} \\
\bottomrule
\end{tabular}
}
\label{WT2 results2}
\end{table}
Tabs. \ref{PTB results2} and \ref{WT2 results2} list the validation and test perplexities
after the training step, fine-tuning step, and dynamic evaluation.
We can see that the validation and test perplexities of sigsoftmax are significantly
reduced by the dynamic evaluation.
Since the dynamic evaluation adapts models to recent sequences,
these results imply that the high expressive power of sigsoftmax
enabled the model to more flexibly adapt to the validation and test data in the dynamic evaluation. 
In addition, under the conditions tuned for softmax in \cite{merity2018regularizing},
 the sigsoftmax-based model might have slightly overfitted to the training data due to the high expressive power.
 We observed that training perplexities of sigsoftmax are smaller 
 than those of softmax at the last epochs.\footnote{Note that models at the last epochs were not used for evaluation
  since their validation perplexities were not the lowest in the training.}
 By tuning the hyper parameters for regularization methods such as dropout rates,
sigsoftmax can achieve better performance.
 \begin{figure}[t]
\centering
\includegraphics[width=13.8cm]{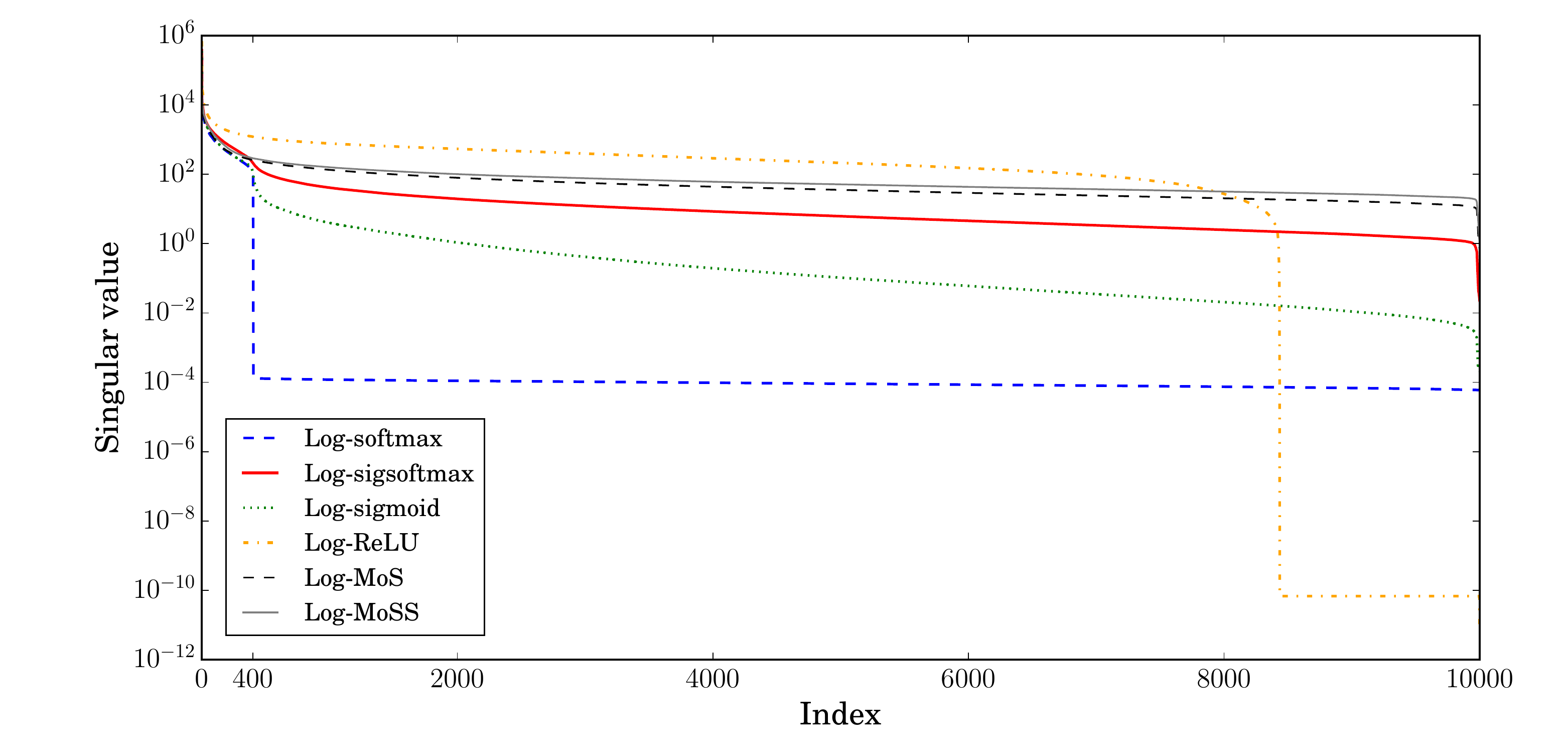}
\caption{Singular values of $\hat{\bm{A}}$ on PTB test set.}
\label{RanksPlot1}
\end{figure}
\begin{figure}[tbp]
\centering
\includegraphics[width=13.8cm]{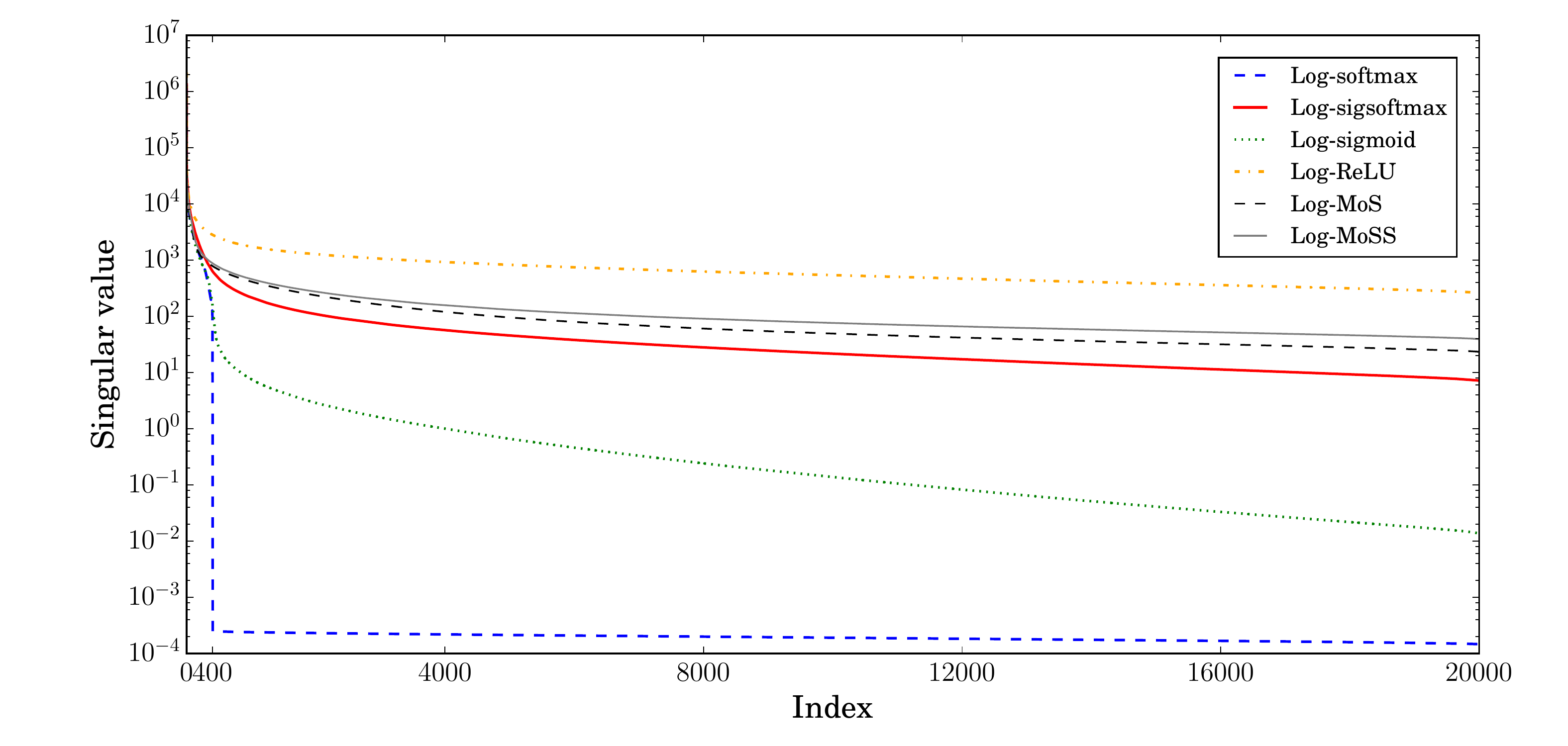}
\caption{Top 20,000 singular values of $\hat{\bm{A}}$ on WT2 test set.}
\label{RanksPlot2}
\end{figure}

 \rfig{RanksPlot1} shows the singular values of $\hat{\bm{A}}$ on PTB,
 and \rfig{RanksPlot2} shows the top 20,000 singular values of $\hat{\bm{A}}$ 
 on WT2.
 On both datasets, singular values of softmax significantly decrease at 403.
 In addition, singular values of the ReLU-based function also significantly decrease at 8243 on PTB.
On the other hand, singular values of sigsoftmax, sigmoid, MoS and MoSS smoothly decrease.
Therefore, their ranks might be greater than those in Tab. 3 in the paper.
\end{document}